%% file: main.tex
\def\isarxiv{1} 

\ifdefined\isarxiv
\documentclass[11pt]{article}

\usepackage[numbers]{natbib}

\else

\documentclass{article} 
\usepackage{iclr2025_conference,times}
\fi

\usepackage{amsmath}
\usepackage{amsthm}
\usepackage{amssymb}
\usepackage{algorithm}
\usepackage{subfig}
\usepackage{booktabs}
\usepackage{algpseudocode}
\usepackage{graphicx}
\usepackage{grffile}
\usepackage{wrapfig,epsfig}
\usepackage{url}
\usepackage{xcolor}
\usepackage{epstopdf}

\usepackage{bbm}
\usepackage{dsfont}

\allowdisplaybreaks

\ifdefined\isarxiv

\usepackage{tikz}
\usepackage{hyperref}  
\hypersetup{colorlinks=true,citecolor=blue,linkcolor=blue} 
\usetikzlibrary{arrows}
\usepackage[margin=1in]{geometry}

\else

\usepackage{hyperref}
\usepackage{url}
 
\fi

\graphicspath{{./figs/}}

\theoremstyle{plain}
\newtheorem{theorem}{Theorem}[section]
\newtheorem{lemma}[theorem]{Lemma}
\newtheorem{definition}[theorem]{Definition}

\newtheorem{corollary}[theorem]{Corollary}

\newtheorem{assumption}[theorem]{Assumption}

\newtheorem{fact}[theorem]{Fact}
\newtheorem{remark}[theorem]{Remark}

\newcommand{\R}{\mathbb{R}}

\renewcommand{\d}{\mathrm{d}}

\makeatletter
\newcommand*{\RN}[1]{\expandafter\@slowromancap\romannumeral #1@}
\makeatother

\usepackage{lineno}

\begin{document}

\ifdefined\isarxiv

\date{}

\title{Grams: Gradient Descent with Adaptive Momentum Scaling}

\author{
Yang Cao\thanks{\texttt{ycao4@wyomingseminary.org}. Wyoming Seminary.}
\and
Xiaoyu Li\thanks{\texttt{xiaoyu.li2@student.unsw.edu.au}. University of New South Wales.}
\and
Zhao Song\thanks{\texttt{magic.linuxkde@gmail.com}. Simons Institute for the Theory of Computing, University of California, Berkeley.}
}

\else

\title{Grams: Gradient Descent with Adaptive Momentum Scaling}

\author{
Yang Cao\thanks{\texttt{ycao4@wyomingseminary.org}. Wyoming Seminary.}
\and
Xiaoyu Li\thanks{\texttt{xiaoyu.li2@student.unsw.edu.au}. University of New South Wales.}
\and
Zhao Song\thanks{\texttt{magic.linuxkde@gmail.com}. Simons Institute for the Theory of Computing, University of California, Berkeley.}
}









\fi

\ifdefined\isarxiv
\begin{titlepage}
  \maketitle
  \begin{abstract}
\input{0_abstract}

  \end{abstract}
  \thispagestyle{empty}
\end{titlepage}

{\hypersetup{linkcolor=black}
\tableofcontents
}
\newpage

\else

\maketitle
\begin{abstract}
\input{0_abstract}
\end{abstract}

\fi

\input{1_intro}

\input{2_related_work}

\input{3_grams}

\input{4_exp}

\input{5_conclusion}

\clearpage

\input{999_impact}

\ifdefined\isarxiv
\bibliographystyle{alpha}
\bibliography{ref}

\else
\bibliography{ref}
\bibliographystyle{iclr2025_conference}

\fi

\newpage
\onecolumn

\begin{center}
	\textbf{\LARGE Appendix }
\end{center}
\appendix


\input{10_proofs}

\input{11_exp_details}



\end{document}

%% file: 0_abstract.tex
We introduce \textbf{Gr}adient Descent with \textbf{A}daptive \textbf{M}omentum \textbf{S}caling (\textbf{Grams}), a novel optimization algorithm that decouples the direction and magnitude of parameter updates in deep learning. Unlike traditional optimizers that directly integrate momentum into updates, Grams separates the update direction, derived from current gradients, from momentum, which is used solely for adaptive magnitude scaling. This approach enables Grams to achieve improved loss descent compared to state-of-the-art cautious and momentum-based optimizers. We theoretically demonstrate that Grams descents faster than other state-of-the-art optimizers and establish a global convergence guarantee for Grams. We also validate its effectiveness through extensive empirical evaluations. The results demonstrate Grams’ superior performance, including faster convergence and better generalization, compared to widely-used optimizers such as Adam, Lion, and their cautious variants. Our results highlight Grams' potential as a transformative approach for efficiently training large language models. Code is available at \href{https://github.com/Gunale0926/Grams}{https://github.com/Gunale0926/Grams}.

%% file: 1_intro.tex
\section{Introduction}

Optimization plays a pivotal role in modern machine learning, serving as the cornerstone for training and fine-tuning models across diverse applications. Over the past decade, the introduction of adaptive optimizers like Adam~\cite{kb14} and its variant AdamW~\cite{lh17} has significantly shaped the landscape of optimization. These algorithms have become the de facto choices for a variety of tasks, ranging from pre-training Large Language Models (LLMs)~\cite{tli+23} to fine-tuning models for text-to-image diffusion~\cite{rbl+22}. Despite the advent of new methods, AdamW has maintained its dominance, particularly in large-scale training regimes, thanks to its robust convergence properties and general applicability.

The era of LLMs has ushered in unprecedented scaling of model sizes, demanding billions or even trillions of parameters~\cite{aaa+23}. This scaling places an immense burden on computational resources, intensifying the need for efficient optimization strategies. A faster optimizer directly translates to the ability to process more training tokens within a fixed time budget, leading to the development of more capable models~\cite{kmh+20}. This necessity has rekindled interest in identifying optimizers that can surpass AdamW in terms of speed, memory efficiency, and convergence guarantees.

Recent innovations, such as SHAMPOO~\cite{gks18}, Schedule Free~\cite{dym+24}, Lion~\cite{clh+24}, SOAP~\cite{vmz+24}, and ADOPT~\cite{thm+24}, have pushed the boundaries of optimization by introducing novel update rules, momentum mechanisms, and regularization techniques. These methods promise substantial improvements in training efficiency and model performance, particularly in specialized scenarios. The cautious~\cite{lcll24} mechanism addresses optimization challenges by adaptively masking the momentum term $u_t$ to align with the gradient $g_t$, preventing conflicts that hinder training. This approach extends to Adam and Lion, resulting in variants like Cautious Adam (C-Adam) and Cautious Lion (C-Lion).

\begin{figure}[!ht]
\centering
\includegraphics[width=\columnwidth]{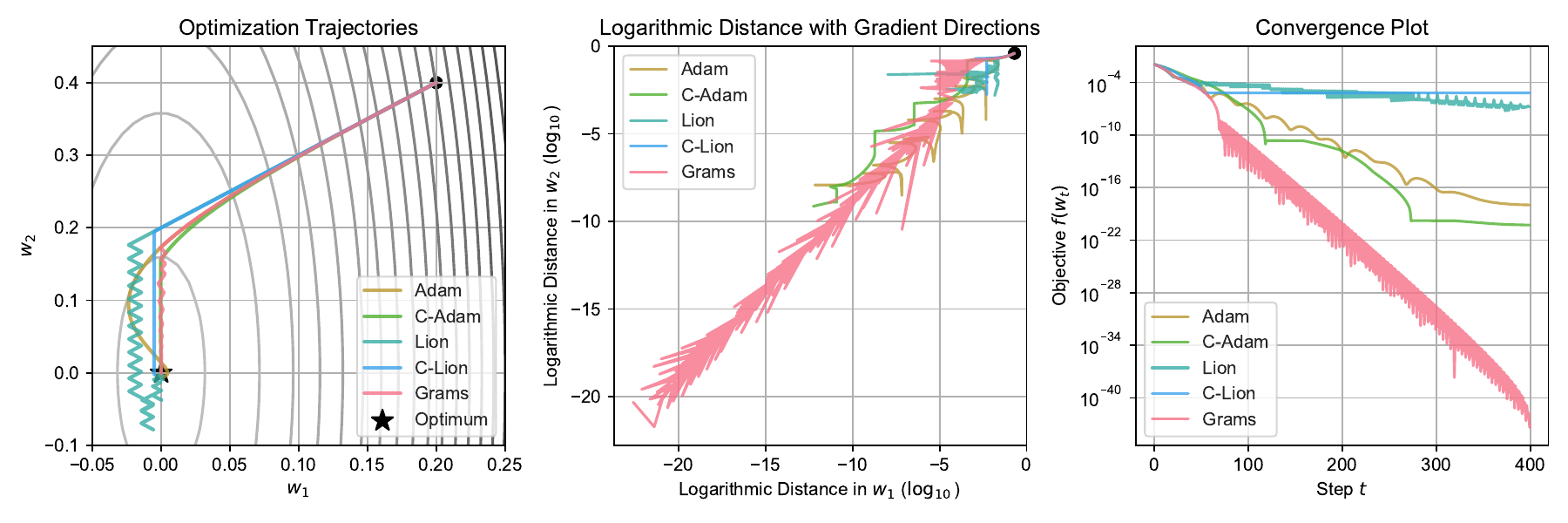}
\caption{Convergence comparison on a simple convex function $f(w) := (0.5 w_1)^2 + (0.1 w_2)^2$. Learning rate $\eta = 0.01$ for Grams, Adam, and C-Adam, and $\eta = 0.001$ for Lion and C-Lion. $\beta_1$ and $\beta_2$ are default values for all optimizers. The graph on the left is the optimizing trajectories; the graph in the middle graph is the distance between current weight and optimum weight; the graph on the right is the training objectives.}
\label{fig:smp_convex}
\end{figure}

In this paper, we propose Gradient Descent with Adaptive Momentum Scaling (Grams), a novel optimization algorithm designed to address the limitations of existing methods. Unlike traditional optimizers that directly couple momentum with gradient updates, Grams decouples the direction and magnitude of parameter updates. This approach allows the update direction to be derived solely from current gradients while momentum is utilized to scale the update magnitude. Such decoupling enhances stability and robustness, particularly in dynamic optimization landscapes.

Figure~\ref{fig:smp_convex} illustrates the superior convergence properties of Grams compared to other state-of-the-art optimizers on a simple convex function. In the left graph, the optimization trajectories of Grams exhibit a combination of characteristics observed in Lion and C-Adam. Specifically, Grams follows a shortcut-like path similar to C-Adam while also demonstrating a zigzagging behavior reminiscent of Lion. However, unlike Lion, which deviates significantly from the optimal path due to its pronounced zigzagging updates, Grams maintains a more controlled trajectory, effectively balancing stability and efficiency during optimization. The middle graph illustrates the logarithmic distance of the weights $w_1$  and $w_2$  from the optimum. Here, Grams consistently demonstrates a faster descent compared to other optimizers, indicating superior efficiency in reducing the distance to the optimal solution. The right graph displays the convergence of the objective function value over training steps, where Grams achieves a notably faster reduction and lower final objective value than its competitors. These results collectively underscore Grams’ ability to navigate the optimization landscape effectively, outperforming traditional and Cautious optimizers in terms of speed and precision, even in a simple convex setting.

Our contributions are summarized as follows:
\begin{itemize}
    \item We introduce the Grams optimizer, which empirically outperforms existing methods such as Adam~\cite{lh17}, Lion~\cite{clh+24}, and their Cautious version~\cite{lcll24}.
    \item We establish theoretical guarantees for Grams, including discrete-time descent analysis and Hamiltonian descent analysis.
    \item We demonstrate the global convergence property of Grams in specific optimization problems under the standard assumptions.
\end{itemize}

By integrating insights from momentum-based methods, adaptive optimizers, and sign-based updates, Grams bridges the gap between theoretical rigor and practical performance, offering a promising direction for scalable and efficient optimization in modern machine learning.

\paragraph{Roadmap.} In Section~\ref{sec:related_work}, we review related work and place our approach in the context of existing optimization methods. In Section~\ref{sec:preli}, we introduce our notation system and outline key preliminary concepts necessary for understanding our method. In Section~\ref{sec:grams}, we present our main contribution, Gradient Descent with Adaptive Momentum Scaling (Grams), and provide theoretical guarantees for its performance. In Section~\ref{sec:exp}, we evaluate the effectiveness of Grams through empirical experiments on both pre-training and fine-tuning tasks, comparing its performance to state-of-the-art optimizers. In Section~\ref{sec:conclusion}, we conclude the paper and discuss future directions to enhance the capabilities of Grams further.

%% file: 2_related_work.tex
\section{Related Work} \label{sec:related_work}

\paragraph{Adam Variants and Memory-Efficient Optimization}
Adam and its numerous variants have been pivotal in addressing optimization challenges across diverse applications~\cite{kb14, log+19}. Among these, AdamW~\cite{log+19} introduced a crucial modification by decoupling weight decay from gradient updates, restoring the original intent of weight regularization. NAdam~\cite{d16} integrated Nesterov momentum, and AdaBelief~\cite{ztd+20} refined the second moment estimation for improved generalization. Adan~\cite{xzl+24} extended these advancements with an additional momentum term, balancing performance with memory overhead. Schedule-free optimizers~\cite{dym+24} have further simplified the optimization process by dynamically adjusting learning rates without pre-defined schedules, enhancing adaptability across tasks. More recent efforts, such as ADOPT~\cite{thm+24}, streamlined first-order momentum updates through normalization.

Memory-efficient strategies have addressed the growing resource demands of large-scale models. AdaFactor~\cite{ss18} factorize second-order statistics, achieving sublinear memory usage. K-Fac~\cite{mg15} approximates the Fisher information matrix using Kronecker-factored representations. Innovations such as fused gradient computation~\cite{lyl+23} and GaLore~\cite{zzc+24} leverage low-rank gradient structures to optimize memory efficiency.

\paragraph{Regularization Techniques}
Regularization plays a critical role in improving generalization and robustness in optimization. Lion~\cite{clh+24} introduced sign-based updates with uniform magnitudes, offering inherent noise regularization~\cite{nvl+17, fkmn21, chg22}. Earlier methods, such as signSGD~\cite{bwaa18}, explored similar ideas but focused on reducing communication costs in distributed optimization. Despite its efficiency, signSGD often underperformed in deep learning tasks, such as ConvNet training, where Lion demonstrated superior performance through advanced momentum mechanisms.

Building on these ideas, the Cautious mechanism~\cite{lcll24} adaptively masks momentum terms to ensure alignment with gradient directions, mitigating conflicts. This approach has led to new variants, including Cautious Adam (C-Adam) and Cautious Lion (C-Lion), which combine regularization benefits with robust convergence guarantees.

\paragraph{Hamiltonian Dynamics in Optimization}
Hamiltonian dynamics provides a robust theoretical framework for understanding momentum-based optimization~\cite{n83, smdh13, ncll24, a24}. The seminal work of~\cite{smdh13} provided a physical interpretation of momentum methods, linking the oscillatory behavior of algorithms like Nesterov’s and Polyak’s methods~\cite{n83} to principles of dynamical systems. While traditional gradient descent guarantees a monotonic decrease in objective function values, momentum-based methods exhibit non-monotonic dynamics that require more advanced analytical tools~\cite{jnj17}. This has motivated the development of Lyapunov-based approaches for convergence analysis in convex optimization~\cite{kbb15, wrj16}.

Recent studies have further formalized these connections by modeling optimization processes as continuous-time ODEs, uncovering inherent Hamiltonian structures~\cite{mpt+18, ncll24}. These insights have significantly enhanced the theoretical understanding of classical momentum-based algorithms and provided a foundation for exploring new optimization frameworks~\cite{a24}. Moreover, Hamiltonian principles have been extended to analyze convergence rates for accelerated methods~\cite{jnj17} and have inspired broader applications in optimization. In parallel, Mirror Descent, while distinct from Hamiltonian dynamics, leverages variational principles and maintains efficiency with a mild dependence on the dimensionality of decision variables, making it well-suited for large-scale problems~\cite{kbb15, trrb23}.

%% file: 3_grams.tex
\section{Preliminaries} \label{sec:preli}

In this section, we outline foundational concepts and notations that will be referenced throughout the paper. In Section~\ref{sec:preli:notation}, we define some useful notations. In Section~\ref{sec:preli:background_opt} provides essential definitions in optimization, which are critical for understanding the theoretical guarantees of our proposed method. 
In Section~\ref{sec:preli:adam}, \ref{sec:preli:lion} and \ref{sec:preli:cautious}, we review key optimizers, including Adam~\cite{lh17}, Lion~\cite{clh+24}, and the Cautious mechanism~\cite{lcll24}. In Section~\ref{sec:preli:ham}, we summarize the Hamiltonian dynamics framework, which provides a theoretical foundation for understanding momentum-based optimization algorithms.

\subsection{Notations}\label{sec:preli:notation}
For two vectors $u, v\in \R^d$, we use $\langle u, v\rangle$ to denote the standard inner product in the Euclidean space. We use $\|u\|_2$ to denote the $\ell_2$-norm of $u$ and use $\|u\|_\infty$ to denote the $\ell_\infty$-norm of $u$. For a matrix $A$, we use $\|A\|_F$ to denote the Frobenius norm of $A$. For a twice differentiable function $f: \R^d \to \R$, we use $\nabla f(x)$ and $\nabla^2 f(x)$ to denote the gradient and Hessian of $f$, respectively. Given a vector $x \in \R^d$, we use ${\bf 1}_{x \geq 0} \in \R^d$ to denote the vector where each entry indicates whether the corresponding entry of $x$ is non-negative, i.e., for each $i \in [d]$, $({\bf 1}_{x \geq 0})_i = 1$ if $x_i \geq 0$, and $({\bf 1}_{x \geq 0})_i = 0$ otherwise. 

\subsection{Backgrounds on Optimization}\label{sec:preli:background_opt}

We define the $L$-smoothness of functions as below.
\begin{definition}[$L$-smooth]\label{def:smooth}
We say that a function $f: \R^d \to \R$ is $L$-smooth if $\|\nabla f(x_1) - \nabla f(x_2) \|_2 \leq L \|x_1 - x_2\|_2 $ for all $x_1, x_2 \in \R^d$.
\end{definition}

We state a common fact of $L$-smooth functions as follow.

\begin{fact}\label{fac:smooth_upper_bound}
    If a function $f: \R^d \to \R$ is $L$-smooth, then we have
    \begin{align*}
        f(x_2) \leq & ~ f(x_1) + \langle \nabla f(x_1), x_2 - x_1 \rangle + \frac{L}{2} \|x_2 - x_1\|_2^2, \\
        f(x_2) \geq & ~ f(x_1) + \langle \nabla f(x_1), x_2 - x_1 \rangle - \frac{L}{2} \|x_2 - x_1\|_2^2.
\end{align*}
\end{fact}

We also define PL-condition as below.

\begin{definition}[PL-condition]\label{def:pl_condition}
A function $f: \mathbb{R}^d \to \mathbb{R}$ satisfies the $\mu$-Polyak–Łojasiewicz (PL) condition with constant $\mu > 0$ if the following inequality holds for all $x \in \mathbb{R}^d$:
\begin{align*}
\|\nabla f(x)\|^2 \geq 2\mu (f(x) - f^*),
\end{align*}
where $f^*$ is the minimum value of the function $f$, i.e., $f^* = \inf_{x \in \mathbb{R}^d} f(x)$.
\end{definition}







\subsection{Sign Function}

We formally define the sign function, which will be used later in our optimizer Grams.

\begin{definition}[Sign function]\label{def:sign}
    Given a vector $a = (a_1, a_2, \dots, a_n) \in \mathbb{R}^n$, the \textit{sign function} of $a$, denoted as $\mathrm{sign}(a)$, is defined component-wise as:
    \begin{align*}
        \mathrm{sign}(a) = (\mathrm{sign}(a_1), \mathrm{sign}(a_2), \dots, \mathrm{sign}(a_n)),
    \end{align*}
    where the scalar sign function $\mathrm{sign}(a_i)$ is given by:
    \begin{align*}
        \mathrm{sign}(a_i) =
        \begin{cases}
        1, & \text{if } a_i > 0, \\
        0, & \text{if } a_i = 0, \\
        -1, & \text{if } a_i < 0.
    \end{cases}
    \end{align*}
\end{definition}




\subsection{Adam Optimizer} \label{sec:preli:adam}

Adam (Adaptive Moment Estimation)~\cite{kb14} is a widely-used optimizer that combines the benefits of RMSprop~\cite{hss12} and momentum by maintaining both first and second moment estimates of the gradients. The algorithm adapts the learning rates for each parameter using these estimates.

\begin{definition}[Adam]\label{def:adam}
The parameter update rule for Adam is given by:
\begin{align*}
m_t := & ~ \beta_1 m_{t-1} + (1-\beta_1)g_t \\
v_t := & ~ \beta_2 v_{t-1} + (1-\beta_2)g_t^2 \\
\widehat{m}_t := & ~ \frac{m_t}{1-\beta_1^t} \\
\widehat{v}_t := & ~ \frac{v_t}{1-\beta_2^t} \\
u_t := & ~ \frac{\widehat{m}_t}{\sqrt{\widehat{v}_t} + \epsilon} \\
w_{t+1} := & ~ w_t - \eta_t u_t,
\end{align*}
where $w_t$ is the weight at time step $t$, $m_t$ and $v_t$ are the first and second momentum estimates respectively, $g_t = \nabla_w\mathcal{L}_t(w_{t-1})$ is the current gradient, $\beta_1$ and $\beta_2$ are decay rates for the moment estimates, $\epsilon$ is a small constant for numerical stability, and $\eta_t$ is the learning rate at step $t$.
\end{definition}


\subsection{Lion Optimizer} \label{sec:preli:lion}

Evolved Sign Momentum (Lion)~\cite{clh+24} is an efficient optimizer that leverages momentum and sign-based updates. Lion's key innovation lies in its update rule, which combines both current and momentum gradients through sign operations.

\begin{definition}[Lion Parameter Update]\label{def:lion}
    The parameter update rule for Lion is given by:
    \begin{align*}
    u_t := & ~ \mathrm{sign}(\beta_1 m_{t-1} + (1-\beta_1)g_t)\\
    w_t := & ~ w_{t-1} - \eta_t \cdot u_t\\
    m_t := & ~ \beta_2m_{t-1} + (1 - \beta_2)g_t,
    \end{align*}
    where $w_t$ is the weight at time step $t$, $m_{t-1}$ is the momentum term, $g_t = \nabla_w\mathcal{L}_t(w_{t-1})$ is the current gradient, $\beta_1$ and $\beta_2$ are the momentum coefficients, $\eta_t$ is the learning rate at step $t$, and $\mathrm{sign}$ is defined in Definition~\ref{def:sign},
\end{definition}

Lion's efficiency stems from its memory-efficient design - it only needs to maintain a single momentum term and operates primarily through sign operations. This makes it particularly suitable for large-scale training where memory constraints are significant. The optimizer has demonstrated strong performance in training large language models and vision transformers, often achieving comparable or better results than Adam while using less memory.

\subsection{Cautious Optimizers} \label{sec:preli:cautious}

Cautious mechanism~\cite{lcll24} addresses a key challenge in optimization dynamics: when the momentum term $u_t$ moves in a different direction from the current gradient $g_t$, it can potentially impede training progress. To mitigate this issue, the Cautious mechanism introduces an adaptive masking mechanism that modifies the momentum term based on its alignment with the gradient direction. Cautious mechanism could apply to Adam and Lion, which form Cautious Adam (C-Adam) and Cautious Lion (C-Lion).

\begin{definition}[Cautious Mechanism Parameter Update] \label{def:cautious_update}
    The general parameter update rule for the Cautious mechanism is given by:
    \begin{align}
        \widehat{u}_t := & ~  u_t \circ {\bf 1}_{u_t \circ g_t \geq 0} \notag \\
        w_t := & ~ w_{t-1} - \eta_t \widehat{u}_t, \label{eq:cautious_update}
    \end{align}
    where $w_t$ is the weight at time step $t$, $\circ$ denotes Hadamard product. For C-Adam, $u_t$ is from Definition~\ref{def:adam}; For C-Lion, $u_t$ is from Definition~\ref{def:lion}. $g_t$ is the current gradient.
\end{definition}

The Cautious mechanism in Definition~\ref{def:cautious_update} modifies the parameter updates to ensure they align with the gradient direction, thereby reducing the risk of adverse updates that could impede convergence. To analyze the impact of this mechanism, we introduce Definition~\ref{def:delta_l}, which quantifies the change in the loss function after an update.

\begin{definition} \label{def:delta_l}
    For any loss function $\mathcal{L}: \R^d \to \R$, we define
    \begin{align*}
        \Delta \mathcal{L}_{w_{t+1}, w_t} := & ~ \mathcal{L}(w_{t+1}) - \mathcal{L}(w_t)
    \end{align*}
    where $w_{t+1}$ is updated from any update rule.
\end{definition}


As shown in \cite{lcll24}, the Cautious mechanism ensures that the updated parameters result in a non-negative inner product with the gradient, leading to a monotonic decrease in the loss function when the step size is sufficiently small. Specifically, using a Taylor approximation, it can be expressed as:
\begin{align*}
    \Delta\mathcal{L}_{w_{t+1},w_t} \approx -\eta_t (u_t \circ g_t)^\top \phi(u_t \circ g_t) \leq 0,
\end{align*}
where $\phi(\cdot)$ represents the alignment mask introduced by the Cautious mechanism. This guarantees that $\mathcal{L}(w_{t+1}) \leq \mathcal{L}(w_t)$, ensuring a decrease in loss.

We formalize that the expected decrease in loss when updating the parameter $w$ from optimization step $t$ to step $t+1$ can be approximated using a first-order Taylor expansion, which indicates the loss function will decrease monotonically when the step size is sufficiently small.
\begin{lemma}[Informal version of Lemma~\ref{lem:delta_l_c}] \label{lem:delta_l_c:informal}
    Suppose that $\mathcal L:\R^d \to \R$ is $L$-smooth.
    Let $\Delta \mathcal{L}_{w_{t+1}^{\mathrm{C}}, w_t}$ be defined in Definition~\ref{def:delta_l}, $w_{t+1}^{\mathrm{C}}$ is updated from $w_t$ using Definition~\ref{def:cautious_update}. Then we have the followings:

    \begin{itemize}
        \item Part 1. It holds that
    \begin{align}
        \Delta \mathcal{L}_{w_{t+1}^{\mathrm{C}}, w_t} 
        \leq & ~ -\eta_t \langle u_t \circ g_t, {\bf 1}_{u_t \circ g_t \geq 0}\rangle + \frac{L \eta_t^2}{2}\| u_t\|_2^2, \label{eq:delta_l:informal}
    \end{align}
    \item Part 2. It holds that
    \begin{align*}
            \Delta \mathcal{L}_{w_{t+1}^{\mathrm{C}}, w_t} 
        \geq & ~ -\eta_t \langle u_t \circ g_t, {\bf 1}_{u_t \circ g_t \geq 0}\rangle.
    \end{align*}
    \item Part 3. If $\eta_t \leq \frac{2}{L\|u_t\|^2_2} \langle u_t \circ g_t, {\bf 1}_{u_t \circ g_t \geq 0}\rangle$, then 
    \begin{align*}
        \Delta \mathcal{L}_{w_{t+1}^{\mathrm{C}}, w_t} \leq 0.
    \end{align*}
    \end{itemize}
\end{lemma}

Building on these findings, Theorem~\ref{thm:ham_c} delves into the Hamiltonian properties of the Cautious mechanism, providing deeper insights into its theoretical guarantees within continuous optimization dynamics.

\subsection{Hamiltonian Descent} \label{sec:preli:ham}

Hamiltonian descent provides a theoretical framework for analyzing momentum-based optimization algorithms by introducing an augmented objective function, the Hamiltonian. This framework allows us to study optimization dynamics through the lens of continuous-time differential equations, linking the monotonic descent of the Hamiltonian function to the stability and convergence of the optimization process. We formalize this concept as Definition~\ref{def:ham_des}, based on the formulation presented in Section 2.1 of~\cite{lcll24}.

\section{Gradient Descent with Adaptive Momentum Scaling} \label{sec:grams}

We propose \emph{Gradient Descent with Adaptive Momentum Scaling} (\textbf{Grams}). Grams decouples the direction and magnitude of the update by using the direction from gradients while scaling it with the norm of momentum. This section formalizes the Grams update rule, introduces its key components, and provides theoretical guarantees in both loss descent and Hamiltonian dynamics for its performance.

\subsection{Definitions}

We define the parameter updating rule of Grams formally as below.

\begin{definition}[Grams Parameter Update]\label{def:grams_update}
    The parameter update rule for Grams is:
    \begin{align}
        m_t := & ~ \beta_1 m_{t-1} + (1-\beta_1)g_t \notag \\
        v_t := & ~ \beta_2 v_{t-1} + (1-\beta_2)g_t^2 \notag \\
        \widehat{m}_t := & ~ \frac{m_t}{1-\beta_1^t} \notag \\
        \widehat{v}_t := & ~ \frac{v_t}{1-\beta_2^t} \notag \\
        u_t := & ~ \frac{\widehat{m}_t}{\sqrt{\widehat{v}_t} + \epsilon} \notag \\
        \widehat{u}_t := & ~ \mathrm{sign}(g_t) \circ |u_t| \notag \\
        w_t := & ~ w_{t-1} - \eta_t \widehat{u}_t, \label{eq:grams_update}
    \end{align}
    where $w_t$ is the weight at time step $t$, $g_t = \nabla_w\mathcal{L}_t(w_{t-1})$ is the current gradient, $|\cdot|$ is element-wise absolute value, $\circ$ denotes Hadamard product, and $\mathrm{sign}(\cdot)$ is defined in Definition~\ref{def:sign}.
\end{definition}

\begin{algorithm}[!ht]
\caption{Gradient Descent with Adaptive Momentum Scaling (Grams)}
\label{alg:grams}
\begin{algorithmic}[1]
\Require parameter $w$, step sizes $\{\eta_t\}$, dampening factors $\beta_1, \beta_2 \in [0,1)$, $\epsilon > 0$, weight decay $\gamma \geq 0$
\State Initialize $t = 0$, $m_0 = v_0 = \mathbf{0}$
\While{$w_t$ not converged}
    \State $t \gets t + 1$
    \State $g_t \gets \nabla_w \mathcal{L}_t(w_{t-1})$
    \State $m_t \gets \beta_1 m_{t-1} + (1-\beta_1)g_t$
    \State $v_t \gets \beta_2 v_{t-1} + (1-\beta_2)g_t^2$
    \State $\widehat{m}_t \gets m_t/(1-\beta_1^t)$
    \State $\widehat{v}_t \gets v_t/(1-\beta_2^t)$
    \State $u_t \gets \widehat{m}_t/(\sqrt{\widehat{v}_t} + \epsilon)$
    \State \textcolor{blue}{$\widehat{u}_t \gets \mathrm{sign}(g_t) \circ |u_t|$}
    \State $w_t \gets w_{t-1} - \eta_t \textcolor{blue}{\widehat{u}_t}$
    \State $w_t \gets w_t - \eta_t\gamma w_t$ \Comment{Add weight decay~\cite{lh17}}
\EndWhile
\end{algorithmic}
\end{algorithm}

\subsection{Loss Descent}

In this subsection, we analyze the loss descent properties of the Grams algorithm. Understanding how the loss function decreases over optimization steps provides insights into the efficiency and stability of the method. Below, we formalize the relationship between the step size, gradients, and the resulting decrease in the loss value, leveraging the $L$-smoothness property of the objective function.

\begin{lemma}[Informal version of Lemma~\ref{lem:delta_l_grams}] \label{lem:delta_l_grams:informal}
    Suppose that $\mathcal L:\R^d \to \R$ is $L$-smooth. Let $\Delta \mathcal{L}_{w_{t+1}^{\mathrm{Grams}}, w_t}$ be defined in Definition~\ref{def:delta_l}, $w_{t+1}^{\mathrm{Grams}}$ is updated from $w_t$ using Eq.~\eqref{eq:grams_update}. Then we have the following:
   \begin{itemize}
       \item Part 1. It holds that
        \begin{align}
        \Delta \mathcal{L}_{w_{t+1}^{\mathrm{Grams}}, w_t} \leq - \eta_t \langle |g_t|, |u_t| \rangle + \frac{L \eta_t^2}{2} \| u_t \|_2^2.  \label{eq:delta_l_grams:informal}
    \end{align}
    \item Part 2. It holds that
    \begin{align*}
        \Delta \mathcal{L}_{w_{t+1}^{\mathrm{Grams}}, w_t} \geq - \eta_t \langle |g_t|, |u_t| \rangle. 
    \end{align*}
        \item Part 3. If $\eta_t \leq \frac{2}{L\|u_t\|^2} \langle |g_t|, |u_t| \rangle$, then we have 
        \begin{align*}
            \Delta \mathcal{L}_{w_{t+1}^{\mathrm{Grams}}, w_t} \leq 0.
        \end{align*}
   \end{itemize}
\end{lemma}

Then, we compare the loss descent between Grams and C-Adam.

\begin{theorem}[Loss Descent Comparison, informal version of Theorem~\ref{thm:delta_loss}] \label{thm:delta_loss:informal}
    Suppose that $\mathcal L: \R^d \to \R$ is $L$-smooth.
    For any parameter vector $w$ at optimization step $t$, let $w_{t}^{\mathrm{Grams}}$ and $w_{t}^{\mathrm{C}}$ be the update of Grams in Definition~\ref{def:grams_update} and Cautious optimizers in Definition~\ref{def:cautious_update}, respectively. If the stepsize $\eta_t$ satisfies
    \begin{align*}
        \eta_t \leq \frac{2}{L\|u_t\|^2} \cdot \min\{\langle u_t \circ g_t, {\bf 1}_{u_t \circ g_t \geq 0} \rangle, \langle u_t \circ g_t, {\bf 1}_{u_t \circ g_t < 0}\rangle\},
    \end{align*}
     then we have
    \begin{align*}
        \Delta \mathcal{L}_{w_{t+1}^{\mathrm{Grams}}, w_t} \leq \Delta \mathcal{L}_{w_{t+1}^{\mathrm{C}}, w_t} \leq 0.
    \end{align*}
\end{theorem}

\begin{remark}
    Theorem~\ref{thm:delta_loss:informal} shows that Grams achieves strictly better descent in the loss landscape in the discrete analysis compared to Cautious optimizers. This theoretical guarantee suggests that Grams may converge faster and achieve better minima in practice.
\end{remark}

\subsection{Hamiltonian Dynamics}

In this subsection, we present the Grams Hamiltonian dynamics, which builds upon the augmented Hamiltonian framework to analyze optimization algorithms. By leveraging this framework, we show that the Grams optimizer achieves a monotonic descent of the Hamiltonian and the loss function, with a descent speed that is provably equal to or faster than C-Adam. This highlights Grams’ efficiency and robustness in dynamic optimization landscapes. The formal definition is provided below.

\begin{definition}[Grams Hamiltonian Dynamics] \label{def:grams_ham}
    We could modify Hamiltonian dynamics with Grams' optimizing scheme,
    \begin{align*}
        \frac{\d}{\d t}w_t := & ~ -\mathrm{sign}(\nabla\mathcal{L}(w_t) \circ |\nabla\mathcal{K}(s_t)| - \Phi_t(\nabla \mathcal{L}(w_t)) \\
        \frac{\d}{\d t}s_t := & ~  \nabla \mathcal{L}(w_t) - \Psi_t( \nabla \mathcal{K}(s_t)),
    \end{align*}
    where $|\cdot|$ denotes element-wise absolute value, $\circ$ is the Hadamard product, and $\Phi_t, \Psi_t$ are scaling functions.
\end{definition}

The convergence properties of Grams within the Hamiltonian dynamics framework are formalized in the theorem below.

\begin{theorem}[Convergence of Grams Hamiltonian Dynamics, informal version of Theorem~\ref{thm:ham_grams}] \label{thm:ham_grams:informal}
Following the dynamics in Definition~\ref{def:grams_ham}, we have
\begin{align*}
   \Delta^{\text{Grams}}_{H}(w_t, s_t) := & ~ \frac{\d}{\d t}H(w_t, s_t) \leq 0,\\
   \Delta^{\text{Grams}}_{\mathcal{L}}(w_t) := & ~  \frac{\d}{\d t}\mathcal{L}(w_t) \leq - \Delta_{\mathcal{L}}(w_t, s_t),
\end{align*}
where $\Delta_{H_t}(w_t, s_t)$ and $\Delta_{\mathcal{L}_t}(w_t, s_t)$ represent the decreasing rates of $H$ and $\mathcal{L}$ in accordance with the system in Definition~\ref{def:ham_des}.
\end{theorem}

Based on this theorem, we compare the convergence rates of Grams and Cautious optimizers in the context of Hamiltonian dynamics. The following theorem demonstrates that Grams achieves a faster or equal rate of loss descent compared to Cautious optimizers, highlighting its efficiency in optimization.

\begin{theorem}[Convergence Comparison of Hamiltonian Dynamics between Grams and Cautious Optimizers, informal version of Theorem~\ref{thm:ham_cmp}] \label{thm:ham_cmp:informal}
    From Theorem~\ref{thm:ham_grams:informal} and \ref{thm:ham_c}, recall $\Delta^{\text{Grams}}_{\mathcal{L}}(w_t)$ and $\Delta^{\text{C}}_{\mathcal{L}}(w_t)$:
    \begin{align*}
        \Delta^{\text{Grams}}_{\mathcal{L}}(w_t) \leq \Delta^{\text{C}}_{\mathcal{L}}(w_t).
    \end{align*}
\end{theorem}

\begin{remark}
    Theorem~\ref{thm:ham_cmp:informal} illustrates the faster loss decreasing speed in the Grams Hamiltonian dynamic system, compared to Cautious's counterpart. 
\end{remark}

Building on this comparison, we now state a corollary from~\cite{lcll24} that establishes the convergence of bounded solutions in Hamiltonian systems to stationary points of the augmented loss.

\begin{corollary} [Corollary 2.4 in~\cite{lcll24}]
Assume that $\langle x, \Psi(x) \rangle$ is positive definite for all $x \in \R^d$, $\Psi(0)=0$, and that $H(w,s) = \mathcal{L}(w) + \mathcal{K}(s)$ is differentiable. Then, the bounded solutions of the original system Eq.~\eqref{equ:hd} converge to a stationary point of $H(w,s)$. Similarly, the bounded solutions of Definition~\ref{def:grams_ham} also converge to a stationary point of $H(w,s)$.
\end{corollary}

\subsection{Global Convergence of Grams} \label{sec:convergence}

In this subsection, we establish the global convergence properties of the Grams optimizer. By analyzing the update rules and assumptions on the optimization landscape, we demonstrate that Grams converges to a stationary point of the objective function. This analysis underscores the optimizer’s robustness and effectiveness in a wide range of optimization scenarios.

\subsubsection{Assumptions}

To ensure theoretical rigor, we base our analysis on the following standard assumptions commonly used in optimization theory. These assumptions define the properties of the loss function and the optimization setting, enabling precise derivations of convergence guarantees.

\begin{assumption}[Lower bound of loss]\label{as:lower_bound}
    The Loss function $\mathcal L:\R^d \to \R$ is differentiable and closed within its open domain $\mathrm{dom}(\mathcal L) \subseteq \mathbb{R}^d$ and is bounded from below, i.e., $\mathcal L^* := \inf_w \mathcal L(w) > -\infty$.
\end{assumption}

\begin{assumption}[Bounded gradient]\label{as:bounded_gradient}
    The Loss function $\mathcal L:\R^d \to \R$ satisfies $\nabla \mathcal L(w) \leq G$ for all $w \in \mathrm{dom}(\mathcal L)$.
\end{assumption}

\begin{assumption}[$L$-smooth]\label{as:smooth}
    The Loss function $\mathcal L:\R^d \to \R$ is $L$-smooth for some $L > 0$.
\end{assumption}

\begin{assumption}[$\mu$-PL-condition]\label{as:pl}
    The Loss function $\mathcal L:\R^d \to \R$ satisfies $\mu$-PL-condition for some $\mu > 0$.
\end{assumption}

\subsubsection{Convergence}

In this subsection, we provide a detailed analysis of the convergence properties of the Grams optimizer. We begin by revisiting the convergence guarantee of the widely-used Adam optimizer as established in~\cite{lrj23}. Using this as a foundation, we extend the analysis to Grams, highlighting its enhanced convergence behavior under the same assumptions.

\begin{lemma}[Convergence of Adam, Section~5.3 in~\cite{lrj23}]\label{lem:convergence_adam}
    Suppose that Assumptions~\ref{as:lower_bound},~\ref{as:bounded_gradient}, and~\ref{as:smooth} hold. Given initial weight $w_1$ with initial optimality gap $\Delta_1 := \mathcal L(w_1) - \mathcal L^* < \infty$, choose an large enough $G$ such that $G \geq \max\{\epsilon, 3\sqrt{L\Delta_1}\}$, a small enough fixed step size $\eta > 0$, and $\beta = \Theta(\eta G^{1/2})$. Consider that the weight $w_t$ is updated by Adam 
    for each $t \in [T]$. Then we have
    \begin{align*}
       \frac{1}{T}\sum_{t=1}^T  \|\nabla \mathcal{L}(w_t)\|_2^2 \leq \frac{8G\Delta_1}{\eta T}.
    \end{align*}
\end{lemma}

The result in Lemma~\ref{lem:convergence_adam} establishes a baseline for the convergence of Adam under standard assumptions. Building on this, we extend the analysis to Grams by leveraging its unique update mechanism, which decouples the direction and magnitude of updates. The following theorem demonstrates that Grams achieves global convergence, meaning that it is guaranteed to reach the optimal objective value from any initial point with finite initial optimality gap. 

\begin{theorem}[Convergence of Grams, informal version of Theorem~\ref{thm:convergence_grams}] \label{thm:convergence_grams:informal}
    Suppose that Assumptions~\ref{as:lower_bound},~\ref{as:bounded_gradient},~\ref{as:smooth} and~\ref{as:pl} hold. Given initial point $w_1$ with initial optimality gap $\Delta_1 := \mathcal L(w_1) - \mathcal L^* < \infty$, choose large an enough $G$ such that $G \geq \max\{\epsilon, 3\sqrt{L\Delta_1}\}$, a small enough fixed step size $\eta > 0$, and $\beta = \Theta(\eta G^{1/2})$. Consider that the weight $w_t$ is updated by Grams (Algorithm~\ref{alg:grams}) for each $t \in [T]$. Then we have
    \begin{align*}
       \mathcal{L}(w_T) - \mathcal{L}^* \leq \frac{4G}{\mu \eta T} (\mathcal{L}(w_1) - \mathcal{L}^*).
    \end{align*}
\end{theorem}


%% file: 4_exp.tex
\section{Empirical Experiments} \label{sec:exp}
We conducted comprehensive experiments across both pre-training and fine-tuning stages to evaluate the performance of our proposed Grams optimizer. Comparisons were made against several baseline optimizers, including Adam~\cite{kb14}, Lion~\cite{clh+24}, C-Adam, C-Lion~\cite{lcll24}, and, in some experiments, RMSprop~\cite{hss12, r16}.

For Lion and C-Lion, we followed the recommendation from~\cite{clh+24}, setting their learning rates to $\frac{1}{10} \times \text{Adam learning rate}$. Additional details and hyperparameters of our experiments can be found in Section~\ref{app:exp}.

\subsection{Pre-Training}
We trained from scratch on the Llama 60M model~\cite{dja+24} using the first $2,048,000$ rows of data from English subset of the C4 dataset~\cite{rsr+20} to assess Grams' optimization capability for Transformer-based~\cite{vsp+17} natural language generation (NLG) tasks. Due to the limited computing resources, we trained $1,000$ steps using constant with warm-up scheduler, in order to simulate the beginning part of regular pre-training.We used the first $10,000$ rows of validation data from the English section of the C4 dataset for evaluation. See 
Table~\ref{tab:exp_nlg} for evaluation results.


\begin{table}[!ht] 
\caption{Evaluation results of Llama 60M pre-training experiments.}
\label{tab:exp_nlg}
\begin{center}
\begin{small}
\begin{sc}
\begin{tabular}{l |c}
\toprule
Optimizer & Perplexity$\downarrow$ \\ 
\midrule
Adam & 49.83 \\  
C-Adam & \underline{43.21} \\  
Lion &   50.25\\  
C-Lion &  53.21 \\ 
Grams (ours) & \textbf{38.60} \\ 
\bottomrule
\end{tabular}
\end{sc}
\end{small}
\end{center}
\end{table}

The evaluation results of the Llama 60M pre-training experiments, as presented in Table~\ref{tab:exp_nlg}, reveal that the Grams optimizer achieves the lowest perplexity (38.60) compared to other state-of-the-art optimizers, including Adam (49.83), C-Adam (43.21), Lion (50.25), and C-Lion (53.21). This substantial reduction in perplexity highlights the effectiveness of Grams in optimizing language model performance. While C-Adam and Lion exhibit improvements over their respective base optimizers, Adam and C-Lion, Grams outperforms all variants, underscoring its ability to enhance convergence and generalization. 
The result demonstrates Grams’ superiority in both training efficiency and model quality for large-scale machine learning tasks.

For computer vision tasks, we trained and evaluated the WideResNet-50-2 model~\cite{zk16} on the CIFAR-10 dataset~\cite{k09}. 
Table~\ref{tab:exp_cv} provides the final accuracy results.


\begin{table}[!ht] 
\caption{Evaluation results of WideResNet-50-2 training experiments from scratch.}
\label{tab:exp_cv}
\begin{center}
\begin{small}
\begin{sc}
\begin{tabular}{l|c}
\toprule
Optimizer & Final Acc$\uparrow$ \\
\midrule
RMSprop & 84.47\% \\ 
Adam & 87.56\% \\ 
C-Adam & 88.78\% \\
Lion & 89.21\% \\
C-Lion & \underline{89.42\%} \\
Grams (ours) & \textbf{90.55\%} \\
\bottomrule
\end{tabular}
\end{sc}
\end{small}
\end{center}
\end{table}

Table~\ref{tab:exp_cv} highlight the performance of various optimizers—RMSprop, Adam, C-Adam, Lion, C-Lion, and Grams—on the WideResNet-50-2 model trained on the CIFAR-10 dataset.
The final accuracy results are presented in Table~\ref{tab:exp_cv}, where Grams achieves the highest accuracy of 90.55\%, surpassing Lion (89.21\%), C-Lion (89.42\%), Adam (87.56\%) and C-Adam (88.78\%). These results emphasize the effectiveness of Grams in accelerating optimization while achieving superior generalization, making it a robust choice for computer vision tasks.

\subsection{Fine-Tuning}

We performed full fine-tuning (FT) experiments on the Llama 3.2 1B model~\cite{dja+24} using the MetaMathQA dataset~\cite{yjs+23}. To evaluate the model, we measured accuracy on the GSM-8K dataset~\cite{ckb+21}. Results are reported in Table~\ref{tab:exp_ft}.

\begin{table}[!ht]
\caption{Evaluation results of Llama 3.2 1B FT experiments.}
\label{tab:exp_ft}
\begin{center}
\begin{small}
\begin{sc}
\begin{tabular}{ l | c }
\toprule
Optimizer & GSM-8K$\uparrow$\\
\midrule
Adam & 48.90\% \\
C-Adam & \underline{49.81\%} \\
Grams (ours) & \textbf{51.02\%} \\
\bottomrule
\end{tabular}
\end{sc}
\end{small}
\end{center}
\end{table}

The results in Table~\ref{tab:exp_ft} showcase the performance of different optimizers during the full FT experiments on the Llama 3.2 1B model using the MetaMathQA dataset. The model’s accuracy was evaluated on the GSM-8K dataset. Among the optimizers, Grams achieved the highest accuracy of 51.02\%, outperforming both Adam (48.90\%) and C-Adam (49.81\%). These results highlight the effectiveness of Grams in fine-tuning tasks, particularly in improving the model’s ability to handle complex datasets like GSM-8K. The superior performance of Grams demonstrates its capacity to achieve better generalization and optimization efficiency in fine-tuning scenarios.

We conducted parameter-efficient fine-tuning (PEFT) experiments on the Llama 3.2 3B model using the SORSA method~\cite{c24} and the first 100,000 rows of data from the MetaMathQA dataset~\cite{yjs+23}. The evaluation was performed on the MATH dataset~\cite{hbk+21}, with the results summarized in Table~\ref{tab:exp_peft}.

\begin{table}[!ht]
\caption{Evaluation results of Llama 3.2 3B PEFT experiments.}
\label{tab:exp_peft}
\begin{center}
\begin{small}
\begin{sc}
\begin{tabular}{ l | c }
\toprule
Optimizer & MATH$\uparrow$ \\
\midrule
Adam &\textbf{17.80\%} \\
C-Adam & 16.62\% \\
Grams (ours) & \textbf{17.80\%} \\
\bottomrule
\end{tabular}
\end{sc}
\end{small}
\end{center}
\end{table}

Grams achieved an accuracy of 17.80\%, matching the performance of Adam and outperforming C-Adam (16.62\%). These results indicate that Grams performs comparably to Adam in PEFT scenarios, maintaining its robust optimization capabilities while offering the additional benefits of parameter efficiency. This consistency further emphasizes Grams’ versatility in various fine-tuning settings.

%% file: 5_conclusion.tex
\section{Conclusion and Future Work}\label{sec:conclusion}

In this paper, we introduced Gradient Descent with Adaptive Momentum Scaling (Grams), a novel optimization algorithm designed to decouple the direction and magnitude of parameter updates. By leveraging this decoupling, Grams demonstrated superior performance in both theoretical convergence guarantees and empirical evaluations, outperforming state-of-the-art optimizers such as Adam~\cite{lh17}, Lion~\cite{clh+24}, and their Cautious variants~\cite{lcll24}. The results across various tasks highlight Grams' potential as a transformative approach for efficiently training large language models.

Grams achieved faster convergence and better generalization in our experiments. These properties make it particularly well-suited for modern applications such as large-scale pre-training and fine-tuning of deep learning models, where efficiency and stability are critical.

Building on the promising results of Grams, future work will focus on integrating ideas from recent advancements such as ADOPT~\cite{thm+24}, Schedule Free~\cite{dym+24}, and SOAP-Muon~\cite{vzm+25} methods. Incorporating the ADOPT and schedule-free learning rate adjustment strategies might improve Grams' robustness and performance across diverse tasks and architectures. By blending these complementary innovations with the core principles of Grams, we aim to develop an even more versatile and efficient optimization framework for large language model training.

%% file: 999_impact.tex


%% file: 10_proofs.tex
\paragraph{Roadmap.} In the appendix, we first provide some useful facts in Section~\ref{app:fac}, which are utilized in the results. Section~\ref{app:loss_des} presents a formal analysis of loss descent for Grams optimizers. In Section~\ref{app:ham}, we illustrate the the property of Grams optimizer in the landscape of Hamiltonian dynamics. In Section~\ref{app:g_conv}, we show the formal proof for the global convergence guarantee of Grams optimizer. Finally, we list the details of our experiments in Section~\ref{app:exp}.

\section{Useful Facts} \label{app:fac}

\begin{fact}\label{fac:hadamard_inner_product}
    Given vectors $a, b, c \in \R^d$, we have
    \begin{align*}
        \langle a, b \circ c\rangle = \langle a \circ b, c \rangle.
    \end{align*}
\end{fact}

\begin{fact} \label{fac:a_b}
    Let two vectors $a, b \in \mathbb{R}^n$, then:
    \begin{align*}
        \langle a, -\mathrm{sign}(a) \circ |b| \rangle = & ~ -\langle |a|, |b| \rangle
    \end{align*}
\end{fact}

\begin{proof}
    For the left side of the equation:
    \begin{align*}
        \langle a, -\mathrm{sign}(a) \circ |b| \rangle = & ~ \sum_{i=1}^n -a_i \mathrm{sign}(a_i) |b|_i\\
        = & ~ - \sum_{i=1}^n |a|_i |b|_i \\
        = & ~ - \langle |a|, |b| \rangle
    \end{align*}
    where the first step comes from the definition of inner product, the second step uses Fact~\ref{fac:a_dot_sign_a}, and the final step uses the definition of inner product again.
\end{proof}

\begin{fact} \label{fac:ab_a_b}
    Let two vectors $a, b \in \mathbb{R}^n$, then:
    \begin{align*}
        \langle a, b \rangle - \langle |a|, |b| \rangle \leq 0.
    \end{align*}
\end{fact}

\begin{proof}
    \begin{align*}
        \langle a, b \rangle - \langle |a|, |b| \rangle = & ~ \sum_{i=1}^n a_ib_i - |a|_i |b|_i \\
        = & ~ \sum_{i=1}^n \begin{cases}
        0 & \text{if } a_i \text{ and } b_i \text{ have the same sign} \\
        -2|a_i||b_i| & \text{if } a_i \text{ and } b_i \text{ have opposite signs}
    \end{cases} \\
    \leq & ~ 0,
    \end{align*}
    where the first step uses the definition of inner product, the second step discusses the only two cases we have for signs, and the final inequality comes from basic algebra.
\end{proof}

\begin{fact} \label{fac:ab_absab_mask}
    Let $x = a \circ b$ be an element-wise product of two vectors $a, b \in \mathbb{R}^n$, then:
    \begin{align*}
        \langle a,b \rangle - \langle |a|, |b| \rangle - \langle a \circ b, \mathbf{1} - \mathbf{1}_{a \circ b > 0}\rangle \leq 0
    \end{align*}
\end{fact}

\begin{proof}
    \begin{align*}
        & \langle a, b \rangle - \langle |a|, |b| \rangle - \langle a \circ b, \mathbf{1} - \mathbf{1}_{a \circ b > 0} \rangle \\
        = & ~ \sum_{i=1}^n a_i b_i - \sum_{i=1}^n |a_i||b_i| - ( \sum_{i=1}^n a_i b_i - \sum_{i: a_i b_i > 0}^n a_i b_i ) \\
        = & ~ \sum^n_{i: a_i b_i > 0} a_i b_i - \sum_{i=1}^n |a_i||b_i|,
    \end{align*}
    where the first step expands the terms, and the second step simplifies by splitting the sum based on the sign of $a_i b_i$. 

    If all $a_i b_i \geq 0$, then $\sum^n_{i: a_i b_i > 0} a_i b_i = \sum_{i=1}^n |a_i||b_i|$, so the expression is $0$. Otherwise, $\sum_{i=1}^n |a_i||b_i| > \sum^n_{i: a_i b_i > 0} a_i b_i$, so the expression is negative.

    Thus,
    \begin{align*}
        \langle a,b \rangle - \langle |a|, |b| \rangle - \langle a \circ b, \mathbf{1}_d - \mathbf{1}_{a \circ b > 0}\rangle = \sum^n_{i: a_i b_i > 0} a_i b_i - \sum_{i=1}^n |a_i||b_i| \leq 0.
    \end{align*}
    The proof is complete.
\end{proof}

\begin{fact} \label{fac:a_dot_sign_a}
    Given a scalar $a \in \mathbb{R}$, we have:
    \begin{align*}
        a \cdot \mathrm{sign}(a) = |a|.
    \end{align*}
\end{fact}

\begin{proof}
    Let $a \in \mathbb{R}$. By Definition~\ref{def:sign}:
    \begin{align*}
        \mathrm{sign}(a) =
        \begin{cases}
            1, & \text{if } a > 0, \\
            0, & \text{if } a = 0, \\
            -1, & \text{if } a < 0.
        \end{cases}
    \end{align*}

    Consider the following cases:
    \begin{itemize}
        \item If $a > 0$, then $\mathrm{sign}(a) = 1$, so:
        \begin{align*}
            a \cdot \mathrm{sign}(a) = a \cdot 1 = a = |a|.
        \end{align*}
        \item If $a = 0$, then $\mathrm{sign}(a) = 0$, so:
        \begin{align*}
            a \cdot \mathrm{sign}(a) = 0 \cdot 0 = 0 = |a|.
        \end{align*}
        \item If $a < 0$, then $\mathrm{sign}(a) = -1$, so:
        \begin{align*}
            a \cdot \mathrm{sign}(a) = a \cdot (-1) = -a = |a|.
        \end{align*}
    \end{itemize}

    Thus, in all cases, $a \cdot \mathrm{sign}(a) = |a|$.
\end{proof}

\begin{fact} \label{fac:a_circ_sign_a}
    Given a vector $a = (a_1, a_2, \dots, a_n) \in \mathbb{R}^n$, we have:
    \begin{align*}
        a \circ \mathrm{sign}(a) = |a|,
    \end{align*}
    where the operations are applied component-wise.
\end{fact}

\begin{proof}
    Let $a = (a_1, a_2, \dots, a_n) \in \mathbb{R}^n$. By Definition~\ref{def:sign}, the $\mathrm{sign}$ function is applied component-wise:
    \begin{align*}
        \mathrm{sign}(a) = (\mathrm{sign}(a_1), \mathrm{sign}(a_2), \dots, \mathrm{sign}(a_n)).
    \end{align*}

    Expanding the Hadamard product $a \circ \mathrm{sign}(a)$ component-wise:
    \begin{align*}
        a \circ \mathrm{sign}(a) = (a_1 \cdot \mathrm{sign}(a_1), a_2 \cdot \mathrm{sign}(a_2), \dots, a_n \cdot \mathrm{sign}(a_n)).
    \end{align*}

    By Fact~\ref{fac:a_dot_sign_a} (the scalar version), for each $i$:
    \begin{align*}
        a_i \cdot \mathrm{sign}(a_i) = |a_i|.
    \end{align*}

    Thus:
    \begin{align*}
        a \circ \mathrm{sign}(a) = (|a_1|, |a_2|, \dots, |a_n|) = |a|,
    \end{align*}
    where the absolute value $|a|$ is applied component-wise.
\end{proof}

\section{Loss Descent} \label{app:loss_des}

\begin{lemma}[Formal version of Lemma~\ref{lem:delta_l_c:informal}] \label{lem:delta_l_c}
    Suppose that $\mathcal L:\R^d \to \R$ is $L$-smooth.
    Let $\Delta \mathcal{L}_{w_{t+1}^{\mathrm{C}}, w_t}$ be defined in Definition~\ref{def:delta_l}, $w_{t+1}^{\mathrm{C}}$ is updated from $w_t$ using Definition~\ref{def:cautious_update}. Then we have the followings:

    \begin{itemize}
        \item Part 1. It holds that
    \begin{align}
        \Delta \mathcal{L}_{w_{t+1}^{\mathrm{C}}, w_t} 
        \leq & ~ -\eta_t \langle u_t \circ g_t, {\bf 1}_{u_t \circ g_t \geq 0}\rangle + \frac{L \eta_t^2}{2}\| u_t\|_2^2, \label{eq:delta_l}
    \end{align}
    \item Part 2. It holds that
    \begin{align*}
            \Delta \mathcal{L}_{w_{t+1}^{\mathrm{C}}, w_t} 
        \geq & ~ -\eta_t \langle u_t \circ g_t, {\bf 1}_{u_t \circ g_t \geq 0}\rangle.
    \end{align*}
    \item Part 3. If $\eta_t \leq \frac{2}{L\|u_t\|^2_2} \langle u_t \circ g_t, {\bf 1}_{u_t \circ g_t \geq 0}\rangle$, then $\Delta \mathcal{L}_{w_{t+1}^{\mathrm{C}}, w_t} \leq 0$.
    \end{itemize}
\end{lemma}
\begin{proof}
    \textbf{Proof of Part 1.} 
    We can show that
    \begin{align}
    \Delta \mathcal{L}_{w_{t+1}^{\mathrm{C}}, w_t} 
    = & ~ \mathcal{L}(w_{t+1}) - \mathcal{L}(w_t)  \notag \\
    \leq & ~  \mathcal{L}(w_t) + \langle g_t, w_{t+1} - w_t\rangle + \frac{L}{2}\|w_{t+1} - w_t\|_2^2 - \mathcal{L}(w_t) \notag\\
    = & ~ \langle g_t, w_{t+1} - w_t\rangle + \frac{L}{2}\|w_{t+1} - w_t\|_2^2 \notag\\
    = & ~ \langle g_t, -\eta_t u_t \circ {\bf 1}_{u_t \circ g_t \geq 0}\rangle + \frac{L}{2}\|\eta_t u_t \circ {\bf 1}_{u_t \circ g_t \geq 0}\|_2^2 \notag\\
    = & ~ -\eta_t \langle u_t \circ g_t, {\bf 1}_{u_t \circ g_t \geq 0}\rangle + \frac{L}{2}\|\eta_t u_t \circ {\bf 1}_{u_t \circ g_t \geq 0}\|_2^2 \notag\\
    \leq & ~ -\eta_t \langle u_t \circ g_t, {\bf 1}_{u_t \circ g_t \geq 0}\rangle + \frac{L \eta_t^2}{2}\| u_t\|_2^2 \label{eq:delta_lc_upper_bound}
    \end{align}
    where the first step follows from Definition~\ref{def:delta_l}, the second step follows from that $\mathcal L$ is $L$-smooth and Fact~\ref{fac:smooth_upper_bound}, the third step follows from basic algebra,
    the fourth step follows from Definition~\ref{def:cautious_update}, the fifth step follows from Fact~\ref{fac:hadamard_inner_product}, and the last step follows from basic algebra.

    \textbf{Proof of Part 2.} Next, we can show that
    \begin{align}
    \Delta \mathcal{L}_{w_{t+1}^{\mathrm{C}}, w_t} 
    = & ~ \mathcal{L}(w_{t+1}) - \mathcal{L}(w_t)  \notag \\
    \geq & ~  \mathcal{L}(w_t) +\langle g_t, w_{t+1} + w_t\rangle - \frac{L}{2}\|w_{t+1} - w_t\|_2^2 - \mathcal{L}(w_t) \notag\\
    \geq & ~ \langle g_t, w_{t+1} - w_t\rangle \notag\\
    = & ~ \langle g_t, -\eta_t u_t \circ {\bf 1}_{u_t \circ g_t \geq 0}\rangle \notag\\
    = & ~ -\eta_t \langle u_t \circ g_t, {\bf 1}_{u_t \circ g_t \geq 0}\rangle
    \end{align}
    where the first step follows from Definition~\ref{def:delta_l}, the second step follows from that $\mathcal L$ is $L$-smooth and Fact~\ref{fac:smooth_upper_bound}, the third step follows from basic algebra,
    the fourth step follows from Definition~\ref{def:cautious_update}, the last step follows from Fact~\ref{fac:hadamard_inner_product}.

    \textbf{Proof of Part 3.} By rearranging the Eq.~\eqref{eq:delta_lc_upper_bound}, it is clear that
    if $\eta_t \leq \frac{2}{L\|u_t\|^2_2} \langle u_t \circ g_t, {\bf 1}_{u_t \circ g_t \geq 0}\rangle$, then we have $\Delta \mathcal{L}_{w_{t+1}^{\mathrm{C}}, w_t} \leq 0$.
\end{proof}

\begin{lemma}[Formal version of Lemma~\ref{lem:delta_l_grams:informal}] \label{lem:delta_l_grams}
    Suppose that $\mathcal L:\R^d \to \R$ is $L$-smooth. Let $\Delta \mathcal{L}_{w_{t+1}^{\mathrm{Grams}}, w_t}$ be defined in Definition~\ref{def:delta_l}, $w_{t+1}^{\mathrm{Grams}}$ is updated from $w_t$ using Eq.~\eqref{eq:grams_update}. Then we have the following:
   \begin{itemize}
       \item Part 1. It holds that
        \begin{align}
        \Delta \mathcal{L}_{w_{t+1}^{\mathrm{Grams}}, w_t} \leq - \eta_t \langle |g_t|, |u_t| \rangle + \frac{L \eta_t^2}{2} \| u_t \|_2^2.  \label{eq:delta_l_grams}
    \end{align}
    \item Part 2. It holds that
    \begin{align*}
        \Delta \mathcal{L}_{w_{t+1}^{\mathrm{Grams}}, w_t} \geq - \eta_t \langle |g_t|, |u_t| \rangle. 
    \end{align*}
        \item Part 3. If $\eta_t \leq \frac{2}{L\|u_t\|^2} \langle |g_t|, |u_t| \rangle$, then we have $\Delta \mathcal{L}_{w_{t+1}^{\mathrm{Grams}}, w_t} \leq 0$.
   \end{itemize}
\end{lemma}

\begin{proof}
    \textbf{Proof of Part 1.} 
    We can show that
    \begin{align}
    \Delta \mathcal{L}_{w_{t+1}^{\mathrm{Grams}}, w_t} 
    = & ~ \mathcal{L}(w_{t+1}) - \mathcal{L}(w_t)  \notag \\
    \leq & ~  \mathcal{L}(w_t) + \langle g_t, w_{t+1} - w_t\rangle + \frac{L}{2}\|w_{t+1} - w_t\|_2^2 - \mathcal{L}(w_t) \notag\\
    = & ~ \langle g_t, w_{t+1} - w_t\rangle + \frac{L}{2}\|w_{t+1} - w_t\|_2^2 \notag\\
    = & ~  \langle g_t, -\eta_t \cdot \mathrm{sign}(g_t) \circ |u_t| \rangle + \frac{L}{2}\|\eta_t \cdot \mathrm{sign}(g_t) \circ |u_t| \|_2^2 \notag\\
    = & ~ -\eta_t \langle g_t \circ \mathrm{sign}(g_t), |u_t|\rangle + \frac{L}{2}\|\eta_t u_t\|_2^2 \notag\\
    \leq & ~ -\eta_t \langle |g_t|, |u_t| \rangle + \frac{L \eta_t^2}{2}\| u_t\|_2^2 \label{eq:delta_lg_upper_bound}
    \end{align}
    where the first step follows from Definition~\ref{def:delta_l}, the second step follows from that $\mathcal L$ is $L$-smooth and Fact~\ref{fac:smooth_upper_bound}, the third step follows from basic algebra,
    the fourth step follows from Definition~\ref{def:grams_update}, the fifth step follows from the Fact~\ref{fac:hadamard_inner_product}, and the last step follows from $g_t \circ \mathrm{sign}(g_t) = |g_t|$.

    \textbf{Proof of Part 2.} Next, we can show that
    \begin{align}
    \Delta \mathcal{L}_{w_{t+1}^{\mathrm{Grams}}, w_t} 
    = & ~ \mathcal{L}(w_{t+1}) - \mathcal{L}(w_t)  \notag \\
    \geq & ~  \mathcal{L}(w_t) +\langle g_t, w_{t+1} + w_t\rangle - \frac{L}{2}\|w_{t+1} - w_t\|_2^2 - \mathcal{L}(w_t) \notag\\
    \geq & ~ \langle g_t, w_{t+1} - w_t\rangle \notag\\
    = & ~ \langle g_t, -\eta_t \cdot \mathrm{sign}(g_t) \circ |u_t| \rangle \notag\\
    = & ~ -\eta_t \langle |g_t|, |u_t|\rangle
    \end{align}
    where the first step follows from Definition~\ref{def:delta_l}, the second step follows from that $\mathcal L$ is $L$-smooth and Fact~\ref{fac:smooth_upper_bound}, the third step follows from basic algebra,
    the fourth step follows from Definition~\ref{def:grams_update}, the last step follows from the Fact~\ref{fac:hadamard_inner_product} and Fact~\ref{fac:a_circ_sign_a}.

    \textbf{Proof of Part 3.} By rearranging the Eq.~\eqref{eq:delta_lg_upper_bound}, it is clear that
    if $\eta_t \leq \frac{2}{L\|u_t\|^2_2} \langle |g_T|, |u_t| \rangle$, then we have $\Delta \mathcal{L}_{w_{t+1}^{\mathrm{Grams}}, w_t} \leq 0$.
\end{proof}

\begin{theorem}[Loss Descent Comparison, formal version of Theorem~\ref{thm:delta_loss:informal}] \label{thm:delta_loss}
    Suppose that $\mathcal L: \R^d \to \R$ is $L$-smooth.
    For any parameter vector $w$ at optimization step $t$, let $w_{t}^{\mathrm{Grams}}$ and $w_{t}^{\mathrm{C}}$ be the update of Grams in Definition~\ref{def:grams_update} and Cautious optimizers in Definition~\ref{def:cautious_update}, respectively. If the stepsize $\eta_t$ satisfies
    \begin{align*}
        \eta_t \leq \frac{2}{L\|u_t\|^2} \cdot \min\{\langle u_t \circ g_t, {\bf 1}_{u_t \circ g_t \geq 0} \rangle, \langle u_t \circ g_t, {\bf 1}_{u_t \circ g_t < 0}\rangle\},
    \end{align*}
     then we have
    \begin{align*}
        \Delta \mathcal{L}_{w_{t+1}^{\mathrm{Grams}}, w_t} \leq \Delta \mathcal{L}_{w_{t+1}^{\mathrm{C}}, w_t} \leq 0.
    \end{align*}
\end{theorem}

\begin{proof}
We define the index sets:
\begin{align*}
I^+ = & ~ \{i \in [d] : u_{t,i}, g_{t,i} \geq 0\}; \\
I^- = & ~ \{i \in [d] : u_{t,i}, g_{t,i} < 0\}.
\end{align*}

By Part 1. of Lemma~\ref{lem:delta_l_grams}, we have
\begin{align}
    \Delta \mathcal{L}_{w_{t+1}^{\mathrm{Grams}}, w_t} \leq - \eta_t \langle |g_t|, |u_t| \rangle + \frac{L \eta_t^2}{2} \| u_t \|_2^2.  \label{eq:delta_l_grams_tmp}
\end{align}

By Part 2. of Lemma~\ref{lem:delta_l_c}, we have
\begin{align}
    \Delta \mathcal{L}_{w_{t+1}^{\mathrm{C}}, w_t} 
    \geq & ~ -\eta_t \langle u_t \circ g_t, {\bf 1}_{u_t \circ g_t \geq 0}\rangle. \label{eq:delta_l_c_tmp}
\end{align}

Then we can show that
\begin{align*}
    \Delta \mathcal{L}_{w_{t+1}^{\mathrm{Grams}}, w_t} - \Delta \mathcal{L}_{w_{t+1}^{\mathrm{C}}, w_t} \leq & ~ -\eta_t\langle |g_t|, |u_t| \rangle + \eta_t \langle u_t \circ g_t,{\bf 1}_{u_t \circ g_t \geq 0} \rangle + \frac{L\eta_t^2}{2}\|u_t\|^2_2\\
    = & ~ -\eta_t\sum_{i=1}^d |u_{t, i}||g_{t, i}| + \eta_t\sum_{i \in I^+} u_{t, i}g_{t, i} + \frac{L\eta_t^2}{2}\|u_t\|^2_2\\
    = & ~ -\eta_t\sum_{i \in I^+} |u_{t, i}||g_{t, i}| - \eta_t\sum_{i \in I^-} |u_{t, i}||g_{t, i}| + \eta_t\sum_{i \in I^+} u_{t, i}g_{t, i}+ \frac{L\eta_t^2}{2}\|u_t\|^2_2 \\
    = & ~ -\eta_t\sum_{i \in I^+} u_{t, i}g_{t, i} - \eta_t\sum_{i \in I^-} |u_{t, i}||g_{t, i}| + \eta_t\sum_{i \in I^+} u_{t, i}g_{t, i}+ \frac{L\eta_t^2}{2}\|u_t\|^2_2 \\
    = & ~  - \eta_t\sum_{i \in I^-} |u_{t, i}||g_{t, i}| + \frac{L\eta_t^2}{2}\|u_t\|^2_2 
\end{align*}
where the first step follows from Eq.~\eqref{eq:delta_l_c_tmp} and Eq.~\eqref{eq:delta_l_grams_tmp}, the second step expands vectors element-wise, the third step follows from that $[d]$ is the disjoint union of $I^+$ and $I^-$, the fourth step follows from that $|u_{t, i}||g_{t, i}| = u_{t, i}g_{t, i}$ for $i \in I^+$, and the last step follows from basic algebra.

To ensure $\Delta \mathcal{L}_{w_{t+1}^{\mathrm{Grams}}, w_t} - \Delta \mathcal{L}_{w_{t+1}^{\mathrm{C}}, w_t} \leq 0$, it suffices to have
\begin{align*}
    - \eta_t\sum_{i \in I^-} |u_{t, i}||g_{t, i}| + \frac{L\eta_t^2}{2}\|u_t\|^2_2  \leq 0.
\end{align*}
Rearranging the above inequality gives
\begin{align*}
    \eta_t \leq &~ \frac{2}{L\|u_t\|_2^2}\sum_{i \in I^-} |u_{t,i}||g_{t,i}| \\ = &~ \frac{2}{L\|u_t\|_2^2} \langle g_t \circ u_t, {\bf 1}_{u_t \circ g_t < 0}),
\end{align*}
where the last step follows from the definition of $I^-$ and basic algebra.

Note that by Part 3 of Lemma~\ref{lem:delta_l_c}, if $\eta_t \leq \frac{2}{L\|u_t\|_2^2} \langle g_t \circ u_t, {\bf 1}_{g_t \circ u_t \geq 0} \rangle$, we have $\mathcal{L}_{w_{t+1}^{\mathrm{C}}, w_t} \leq 0$. 
\end{proof}

\section{Hamiltonian Dynamics} \label{app:ham}

\begin{definition}[Section 2.1 from~\cite{lcll24}] \label{def:ham_des}
Momentum-based algorithms can be typically viewed as monotonic descending algorithms on an augmented loss $H(W, S)$, which satisfies $\min_{S}H(W,S) = \mathcal{L}(W),$ so that minimizing $\mathcal{L}(W)$ is equivalent to minimizing $H(W,S)$. A typical choice is
\begin{align*}
    H(w,s) = \mathcal{L}(w) + \mathcal{K}(s),
\end{align*}
where $\mathcal{K}(\cdot)$ is any lower bounded function. 
The continuous-time form of most momentum-based algorithms can be written into a Hamiltonian descent form:
\begin{align}
    \frac{\d}{\d t}w_t = & ~ -\nabla\mathcal{K}(s_t) - \Phi_t(\nabla\mathcal{L}(w_t)) \notag \\
    \frac{\d}{\d t}s_t = & ~ \nabla\mathcal{L}(w_t) - \Psi_t(\nabla\mathcal{K}(s_t))  \label{equ:hd}
\end{align}
where  $H(W, S)$ is a Hamiltonian (or Lyapunov) function 
that satisfies 
\begin{align*}
    \min_{S} H(W, S) = \mathcal{L}(W), ~~~~\forall W, 
\end{align*}
so that minimizing $\mathcal{L}(W)$ reduces to minimizing $H(W, S)$;
and $\Phi(\cdot), \Psi(\cdot)$ are two monotonic mappings satisfying 
\begin{align*}
\langle x,  \Phi(x)\rangle\geq0, && 
\langle x,  \Psi(x)\rangle\geq0, &&
\forall x \in X.
\end{align*}
With $\Phi(X) = \Psi(X) = 0$, the system in \text{\eqref{equ:hd}} reduces to the standard Hamiltonian system that keeps  $H(W_t, S_t) = const$ along the trajectory. 
When adding the descending components with $\Phi$ and $\Psi$, the system then keeps $H(W, S)$ monotonically decreasing:
\begin{align*}
    \frac{\d}{\d t}H(w_t,s_t) = \Delta_H(w_t,s_t) \leq 0,
\end{align*}
where 
\begin{align}
     \Delta_H(w_t,s_t) := -\langle x, \Phi(x)\rangle - \langle x, \Psi(x) \rangle. \label{eq:hm_delta_h}
\end{align}

On the other hand, $\mathcal{L}(w)$, which is the true objective, is not necessarily decreasing monotonically.
\begin{align*}
    \frac{\d}{\d t}\mathcal{L}(w_t) = -\Delta_\mathcal{L}(w_t,s_t),
\end{align*}
where
\begin{align}
    \Delta_\mathcal{L}(w_t,s_t) := \langle \nabla \mathcal{L}(w_t), \nabla \mathcal{K}(s_t)\rangle + \langle \nabla \mathcal{L}(w_t), \Phi_t(\nabla \mathcal{L}(w_t)\rangle. \label{eq:hm_delta_l}
\end{align}
\end{definition}

\begin{theorem}[Theorem 2.3 in~\cite{lcll24}]\label{thm:ham_c}
    For Hamiltonian dynamics of Cautious optimizer (in Definition~\ref{def:cautious_update}), we have:
    \begin{align*}
        \Delta^{\text{C}}_{H}(w_t, s_t) := \frac{\d}{\d t} H(w_t,s_t) = & ~ \langle x_t,  \mathbf{1} - \mathbf{1}_{x_t > 0} \rangle - \Delta_{H}(w_t,s_t).\\
        \Delta^{\text{C}}_{\mathcal{L}}(w_t) := \frac{\d}{\d t} \mathcal{L}(w_t) = & ~ -\langle x_t, \mathbf{1}_{x_t > 0} \rangle - \langle \nabla \mathcal{L}(w_t), \Phi_t(\nabla\mathcal{L}(w_t)) \rangle \\
        = & ~ \langle x_t, \mathbf{1} - \mathbf{1}_{x_t > 0} \rangle - \Delta_{\mathcal{L}}(w_t,s_t).
    \end{align*}
    where $\Delta_{H_t}(w_t, s_t)$ and $\Delta_{\mathcal{L}_t}(w_t)$ represent the decreasing rates of $H$ and $\mathcal{L}$ in accordance with the system in Definition~\ref{def:ham_des}.
    
    Hence:
\begin{itemize}
   \item If $\langle x_t,(\mathbf{1}_d - \mathrm{sign}(x_t)) \rangle \leq 0$ for any $x \in \mathbb{R}^d$, then both $H$ and $\mathcal{L}$ decrease faster than the original system:
   \begin{align*}
       \Delta^{\text{C}}_{H}(w_t, s_t) \leq & ~ -\Delta_{H_t}(w_t, s_t) \leq 0, \\
       \Delta^{\text{C}}_{\mathcal{L}}(w_t) \leq & ~ -\Delta_{\mathcal{L}_t}(w_t, s_t).
   \end{align*}
   
   \item If $\langle x_t,\mathrm{sign}(\nabla\mathcal{L}(w_t)) \rangle \geq 0$ for any $x \in \mathbb{R}^d$, then $\mathcal{L}$ decreases monotonically:
   \begin{align*}
       \Delta^{\text{C}}_{\mathcal{L}}(w_t) \leq 0.
   \end{align*}
\end{itemize}
\end{theorem}

\begin{theorem}[Convergence of Grams Hamiltonian Dynamics, formal version of Theorem~\ref{thm:ham_grams:informal}] \label{thm:ham_grams}
Following the dynamics in Definition~\ref{def:grams_ham}, we have
\begin{align*}
   \Delta^{\text{Grams}}_{H}(w_t, s_t) := & ~ \frac{\d}{\d t}H(w_t, s_t) \leq 0,\\
   \Delta^{\text{Grams}}_{\mathcal{L}}(w_t) := & ~ \frac{\d}{\d t}\mathcal{L}(w_t) \leq - \Delta_{\mathcal{L}}(w_t, s_t),
\end{align*}
where $\Delta_{H_t}(w_t, s_t)$ and $\Delta_{\mathcal{L}_t}(w_t, s_t)$ represent the decreasing rates of $H$ and $\mathcal{L}$ in accordance with the system in Definition~\ref{def:ham_des}.
\end{theorem}

\begin{proof}
    Recall Eq.~\eqref{eq:hm_delta_h} and \eqref{eq:hm_delta_l}:
    \begin{align*}
    \Delta_H(w_t,s_t) := & ~ \langle \nabla \mathcal{L}(w_t), \Phi(\nabla \mathcal{L}(w_t))\rangle + \langle \mathcal{K}(s_t), \Psi(\mathcal{K}(s_t)) \rangle \\
    \Delta_\mathcal{L}(w_t,s_t) := & ~ \langle \nabla \mathcal{L}(w_t), \nabla \mathcal{K}(s_t) \rangle + \langle \nabla \mathcal{L}(w_t), \Phi_t(\nabla \mathcal{L}(w_t) \rangle.
    \end{align*}
    Following the dynamics in Definition~\ref{def:grams_ham}, we can calculate the derivative of $H(w_t,s_t)$ with respect to $t$:
    \begin{align*}
        \Delta^{\text{Grams}}_{H}(w_t, s_t) = & ~ \langle \nabla \mathcal{L}(w_t), \frac{\d}{\d t}w_t \rangle + \langle \nabla \mathcal{K}(s_t), \frac{\d}{\d t}s_t \rangle \\
        = & ~  \langle \nabla \mathcal{L}(w_t), -\mathrm{sign}(\nabla \mathcal{L}(w_t)) \circ |\nabla\mathcal{K}(s_t)| - \Phi_t(\nabla \mathcal{L}(w_t))\rangle \\ & ~ + \langle \mathcal{K}(s_t), \nabla \mathcal{L}(w_t) - \Psi_t( \nabla \mathcal{K}(s_t))\rangle \\
        = & ~ \langle \nabla \mathcal{L}(w_t), -\mathrm{sign}(\nabla \mathcal{L}(w_t)) \circ |\nabla\mathcal{K}(s_t)| \rangle + \langle \nabla \mathcal{K}(s_t), \nabla \mathcal{L}(w_t)\rangle - \langle \nabla \mathcal{L}(w_t), \Phi_t(\nabla \mathcal{L}(w_t))\rangle \\
        & ~ - \langle \nabla \mathcal{K}(s_t), \Psi_t(\nabla \mathcal{K}(s_t))\rangle \\
        = & ~ \langle \nabla \mathcal{L}(w_t), \nabla \mathcal{K}(s_t) \rangle - \langle |\nabla \mathcal{L}(w_t)|, |\nabla \mathcal{K}(w_t)| \rangle - \Delta_H(w_t,s_t) \\
        \leq & ~ 0,
    \end{align*}
    where the first step follows from the chain rule for the time derivative of the Hamiltonian $H$, the second step substitutes the dynamics from Definition~\ref{def:grams_ham}, the third step separates the inner products for clearer analysis, the fourth step follows the definition of $\Delta H(w_t,s_t)$ and Fact~\ref{fac:a_b}, and the last step follows Fact~\ref{fac:ab_a_b}, and $-\Delta H(w_t,s_t) \leq 0$.

    Then, we calculate the derivative of $\mathcal{L}(w_t)$ with respect to $t$.
    \begin{align*}
        \Delta^{\text{Grams}}_{\mathcal{L}}(w_t) = & ~ \langle \nabla \mathcal{L}(w_t), -\mathrm{sign}(\nabla \mathcal{L}(w_t)) \circ |\nabla\mathcal{K}(s_t)| - \Phi_t(\nabla \mathcal{L}(w_t)) \rangle \\
        = & ~ \langle \nabla \mathcal{L}(w_t), -\mathrm{sign}(\nabla \mathcal{L}(w_t)) \circ |\nabla\mathcal{K}(s_t)| \rangle - \langle \nabla \mathcal{L}(w_t), \Phi_t(\nabla \mathcal{L}(w_t)) \rangle \\
        = & ~ - \langle |\nabla \mathcal{L}(w_t)|, |\nabla \mathcal{K}(w_t)| \rangle - \langle \nabla \mathcal{L}(w_t), \Phi_t(\nabla \mathcal{L}(w_t)) \rangle \\
        = & ~ \langle \nabla \mathcal{L}(w_t), \nabla \mathcal{K}(s_t) \rangle - \langle |\nabla \mathcal{L}(w_t)|, |\nabla \mathcal{K}(w_t)| \rangle \\
        & ~ - (\langle \nabla \mathcal{L}(w_t), \Phi_t(\nabla \mathcal{L}(w_t)) \rangle + \langle \nabla \mathcal{L}(w_t), \nabla \mathcal{K}(s_t) \rangle) \\
        = & ~ \langle \nabla \mathcal{L}(w_t), \nabla \mathcal{K}(s_t) \rangle - \langle |\nabla \mathcal{L}(w_t)|, |\nabla \mathcal{K}(w_t)| \rangle - \Delta_{\mathcal{L}}(w_t,s_t)
    \end{align*}
    where the first step follows from the chain rule, and the second step separates the inner products. The third step follows Fact~\ref{fac:a_b}, the fourth step adds and subtracts the term $\langle \nabla\mathcal{L}(w_t), \nabla\mathcal{K}(s_t) \rangle$ simultaneously, the fifth step follows the definition of $\Delta_{\mathcal{L}}(w_t,s_t)$ from Eq.~\eqref{eq:delta_l:informal}.

    Since we know $\langle \nabla \mathcal{L}(w_t), \nabla \mathcal{K}(s_t) \rangle - \langle |\nabla \mathcal{L}(w_t)|, |\nabla \mathcal{K}(w_t)| \rangle \leq 0$ from Fact~\ref{fac:ab_a_b},
    \begin{align*}
        \langle \nabla \mathcal{L}(w_t), \nabla \mathcal{K}(s_t) \rangle - \langle |\nabla \mathcal{L}(w_t)|, |\nabla \mathcal{K}(w_t)| \rangle \leq - \Delta_{\mathcal{L}}(w_t,s_t)
    \end{align*}
Thus we complete the proof.
\end{proof}

\begin{theorem}[Convergence Comparison of Hamiltonian Dynamics between Grams and Cautious Optimizers, formal version of Theorem~\ref{thm:ham_cmp:informal}] \label{thm:ham_cmp}
    From Theorem~\ref{thm:ham_grams} and \ref{thm:ham_c}, recall $\Delta^{\text{Grams}}_{\mathcal{L}}(w_t)$ and $\Delta^{\text{C}}_{\mathcal{L}}(w_t)$:
    \begin{align*}
        \Delta^{\text{Grams}}_{\mathcal{L}}(w_t) \leq \Delta^{\text{C}}_{\mathcal{L}}(w_t).
    \end{align*}
\end{theorem}

\begin{proof}
    We calculate the difference between $\Delta^{\text{Grams}}_{\mathcal{L}}(w_t)$ and $\Delta^{\text{C}}_{\mathcal{L}}(w_t)$:
    \begin{align*}
         \Delta^{\text{Grams}}_{\mathcal{L}}(w_t) - \Delta^{\text{C}}_{\mathcal{L}}(w_t) = & ~ \langle \nabla \mathcal{L}(w_t), \nabla \mathcal{K}(s_t) \rangle - \langle |\nabla \mathcal{L}(w_t)|, |\nabla \mathcal{L}(w_t)| \rangle - \langle x_t, \mathbf{1} - \mathbf{1}_{x_t > 0} \rangle,
    \end{align*}
    where $x_t = \nabla \mathcal{L}(w_t) \circ \nabla \mathcal{K}(s_t)$. 

    By applying Fact~\ref{fac:ab_absab_mask}, we know:
    \begin{align*}
        \langle \nabla \mathcal{L}(w_t), \nabla \mathcal{K}(s_t) \rangle - \langle |\nabla \mathcal{L}(w_t)|, |\nabla \mathcal{K}(s_t)| \rangle - \langle x_t, \mathbf{1} - \mathbf{1}_{x_t > 0} \rangle \leq 0,
    \end{align*}
    with equality if all components of $\nabla \mathcal{L}(w_t) \circ \nabla \mathcal{K}(s_t) \geq 0$. 

    Thus:
    \begin{align*}
        \Delta^{\text{Grams}}_{\mathcal{L}}(w_t) - \Delta^{\text{C}}_{\mathcal{L}}(w_t) \leq 0,
    \end{align*}
    which implies:
    \begin{align*}
        \Delta^{\text{Grams}}_{\mathcal{L}}(w_t) \leq \Delta^{\text{C}}_{\mathcal{L}}(w_t).
    \end{align*}

    Thus we complete the proof.
\end{proof}

\section{Global Convergence} \label{app:g_conv}

\begin{theorem}[Convergence of Grams, formal version of Theorem~\ref{thm:convergence_grams:informal}] \label{thm:convergence_grams}
    Suppose that Assumptions~\ref{as:lower_bound},~\ref{as:bounded_gradient},~\ref{as:smooth} and~\ref{as:pl} hold. Given initial point $w_1$ with initial optimality gap $\Delta_1 := \mathcal L(w_1) - \mathcal L^* < \infty$, choose large an enough $G$ such that $G \geq \max\{\epsilon, 3\sqrt{L\Delta_1}\}$, a small enough fixed step size $\eta > 0$, and $\beta = \Theta(\eta G^{1/2})$. Consider that the weight $w_t$ is updated by Grams (Algorithm~\ref{alg:grams}) for each $t \in [T]$. Then we have
    \begin{align*}
       \mathcal{L}(w_T) - \mathcal{L}^* \leq \frac{4G}{\mu \eta T} (\mathcal{L}(w_1) - \mathcal{L}^*).
    \end{align*}
\end{theorem}

\begin{proof}
    Given initial weight $w_1$, we denote $w'_1, w'_2, \ldots, w'_T$ be the weights updated by Adam where $w'_1 := w_1$. By Lemma~\ref{lem:convergence_adam}, we have
    \begin{align}
    \label{eq:tmp_1}
      \frac{1}{T}\sum_{t=1}^T  \|\nabla \mathcal{L}(w'_t)\|_2^2 \leq \frac{8G\Delta_1}{\eta T}.
    \end{align}
    For each $t \in [T]$, we can show that
    \begin{align}
        \|\nabla \mathcal{L}(w'_t)\|_2^2 \geq & ~ 2\mu (\mathcal L(w'_t) - \mathcal L^*) \notag \\
        = & ~ 2\mu (\mathcal L(w'_t) -\mathcal L(w'_{t-1}) + \mathcal L(w'_{t-1}) - \mathcal L(w'_{t-2}) + \cdots + \mathcal L(w'_2) - \mathcal L(w'_1) + \mathcal L(w'_1)  - \mathcal L^*) \notag \\
        = &~ 2\mu (\Delta \mathcal{L}_{w_{t}', w_{t-1}} + \Delta \mathcal{L}_{w_{t-1}', w_{t-2}} + \cdots + \Delta \mathcal{L}_{w_{2}', w_1} + \mathcal L(w'_1)  - \mathcal L^*)
        \notag \\
        \geq &~ 2\mu (\Delta \mathcal{L}_{w_{t}, w_{t-1}} + \Delta \mathcal{L}_{w_{t-1}, w_{t-2}} + \cdots + \Delta \mathcal{L}_{w_{2}, w_1} + \mathcal L(w'_1)  - \mathcal L^*)
        \notag \\
        = &~ 2\mu (\Delta \mathcal{L}_{w_{t}, w_{t-1}} + \Delta \mathcal{L}_{w_{t-1}, w_{t-2}} + \cdots + \Delta \mathcal{L}_{w_{2}, w_1} + \mathcal L(w_1)  - \mathcal L^*) \notag \\
        \geq & ~ 2\mu (\mathcal L(w_t) -\mathcal L(w_{t-1}) + \mathcal L(w_{t-1}) - \mathcal L(w_{t-2}) + \cdots + \mathcal L(w_2) - \mathcal L(w_1) + \mathcal L(w_1)  - \mathcal L^*) \notag \\
         = & ~ 2\mu (\mathcal L(w_t)- \mathcal L^*) \label{eq:tmp_2},
    \end{align}
where the first step follows from Assumption~\ref{as:pl}, the second step follows from basic algebra, the third step follows from Definition~\ref{def:delta_l}, the fourth step follows form Theorem~\ref{thm:delta_loss}, the fifth step follows from $w'_1 = w_1$, the sixth step follows from Definition~\ref{def:delta_l}, and the last step follows from basic algebra.

Combining Eq.~\eqref{eq:tmp_1} and Eq.~\eqref{eq:tmp_2} gives
\begin{align*}
    \mathcal{L}(w_T) - \mathcal{L}^* \leq \frac{4G}{\mu \eta T} (\mathcal{L}(w_1) - \mathcal{L}^*).
\end{align*}
Thus we complete the proof.

\end{proof}

%% file: 11_exp_details.tex
\section{Experiments Details} \label{app:exp}

For the Lion and C-Lion optimizers, we set the learning rate to $\frac{1}{10} \times \text{Adam learning rate}$, as recommended in~\cite{clh+24}.

\subsection{Pre-Training}
For the pre-training experiments with Llama 3.2 60M~\cite{dja+24}, we used the first $2,048,000$ rows of training data from the English section of the C4 dataset~\cite{rsr+20}. We used the first $10,000$ rows of validation data from the English section of the C4 dataset for evaluation. Table~\ref{tab:hyp_nlg_pre} provides a detailed summary of the hyperparameters employed.

\begin{table}[!ht]
    \centering
    \caption{Hyperparameters for Llama 3.2 60M pre-training experiments.}
    \begin{tabular}{c|c|c}
    \toprule
    Optimizers & \textbf{Grams/AdamW/CAdamW}&\textbf{Lion/CLion}\\
    \midrule
    \multicolumn{3}{c}{\textbf{Training}} \\
    \midrule
    Epoch & 1 & 1 \\
    Learning Rate & 6e-3 & 6e-4 \\
    Weight Decay & 0.0 & 0.0 \\
    Batch Size & 2048 & 2048 \\
    Model Precision & BF16 & BF16 \\
    Mix Precision & BF16\&TF32 & BF16\&TF32 \\
    Scheduler & Constant with warm-up & Constant with warm-up \\
    Warm-up Steps & 50 & 50 \\
    Grad Clipping & 1.0 & 1.0 \\
    $\beta_1$ & 0.9 & 0.9 \\
    $\beta_2$ & 0.95 & 0.95 \\
    $\epsilon$ & 1e-6 & 1e-6 \\
    Seq-len & 256 & 256 \\
    \midrule
    \multicolumn{3}{c}{\textbf{Evaluating}} \\
    \midrule
    Precision & \multicolumn{2}{c}{BF16} \\
    Seq-len & \multicolumn{2}{c}{256} \\
    \bottomrule
    \end{tabular}
    \label{tab:hyp_nlg_pre}
\end{table}

For the computer vision experiments, we used the CIFAR-10 dataset~\cite{k09} to train and evaluate the WideResNet-50-2 model~\cite{zk16}. Table~\ref{tab:hyp_cv_pre} outlines the corresponding hyperparameters.

\begin{table}[!ht]
    \centering
    \caption{Hyperparameters for WideResNet-50-2 pre-training experiments.}
    \begin{tabular}{c|c|c}
    \toprule
    Optimizers & \textbf{Grams/AdamW/CAdamW}&\textbf{Lion/CLion}\\
    \midrule
    \multicolumn{3}{c}{\textbf{Training}} \\
    \midrule
    Epoch & 10 & 10 \\
    Learning Rate & 2e-3 & 2e-4 \\
    Weight Decay & 0.0 & 0.0 \\
    Batch Size & 128 & 128 \\
    Model Precision & FP32 & FP32 \\
    Mix Precision & None & None \\
    Scheduler & Linear & Linear \\
    Warm-up Steps & 100 & 100 \\
    Grad Clipping & 1.0 & 1.0 \\
    $\beta_1$ & 0.9 & 0.9 \\
    $\beta_2$ & 0.999 & 0.99 \\
    $\epsilon$ & 1e-6 & 1e-6 \\
    \midrule
    \multicolumn{3}{c}{\textbf{Evaluating}} \\
    \midrule
    Precision & \multicolumn{2}{c}{FP32} \\
    \bottomrule
    \end{tabular}
    \label{tab:hyp_cv_pre}
\end{table}

\subsection{Fine-Tuning}

For fine-tuning experiments of the Llama 3.2 1B model, Table~\ref{tab:hyp_nlg_ft} provides the detailed hyperparameters.

\begin{table}[!ht]
    \centering
    \caption{Hyperparameters for Llama 3.2 1B fine-tuning experiments.}
    \begin{tabular}{c|c}
    \toprule
    Optimizers & \textbf{Grams/AdamW/CAdamW}\\
    \midrule
    \multicolumn{2}{c}{\textbf{Training}} \\
    \midrule
    Epoch & 1 \\
    Learning Rate & 1e-4 \\
    Weight Decay & 0.0 \\
    Batch Size & 64 \\
    Model Precision & BF16 \\
    Mix Precision & BF16\&TF32 \\
    Scheduler & Cosine \\
    Warm-up Ratio & 0.03  \\
    Grad Clipping & 1.0 \\
    $\beta_1$ & 0.9 \\
    $\beta_2$ & 0.999  \\
    $\epsilon$ & 1e-6 \\
    Seq-len & 512  \\
    \midrule
    \multicolumn{2}{c}{\textbf{Evaluating}} \\
    \midrule
    Precision & BF16 \\
    Seq-len & 1024 \\
    \bottomrule
    \end{tabular}
    \label{tab:hyp_nlg_ft}
\end{table}

For PEFT of the Llama 3.2 3B model, Table~\ref{tab:hyp_nlg_ft} provides the detailed hyperparameters.

\begin{table}[!ht]
    \centering
    \caption{Hyperparameters for Llama 3.2 3B PEFT experiments.}
    \begin{tabular}{c|c}
    \toprule
    Optimizers & \textbf{Grams/AdamW/CAdamW}\\
    \midrule
    \multicolumn{2}{c}{\textbf{Training}} \\
    \midrule
    Epoch & 1 \\
    Learning Rate & 1e-4 \\
    Weight Decay & 0.0 \\
    Batch Size & 128 \\
    Model Precision & BF16 \\
    Mix Precision & BF16\&TF32 \\
    Scheduler & Cosine \\
    Warm-up Ratio & 0.03  \\
    Grad Clipping & 1.0 \\
    $\beta_1$ & 0.9 \\
    $\beta_2$ & 0.999  \\
    $\epsilon$ & 1e-6 \\
    Seq-len & 512  \\
    Rank & 128 \\
    SORSA~\cite{c24} $\gamma$ & 1e-3 \\
    \midrule
    \multicolumn{2}{c}{\textbf{Evaluating}} \\
    \midrule
    Precision & BF16 \\
    Seq-len & 2048 \\
    \bottomrule
    \end{tabular}
    \label{tab:hyp_nlg_peft}
\end{table}